\documentclass[letterpaper, 10 pt, conference]{ieeeconf}  

\IEEEoverridecommandlockouts                              

\UseRawInputEncoding

\usepackage{graphics} 
\usepackage{epsfig} 
\usepackage{times} 
\usepackage{amsmath} 
\usepackage{amssymb}  

\usepackage{graphicx}
\usepackage{multirow}
\usepackage{booktabs}
\usepackage{fancyhdr}
\usepackage{subfigure}
\usepackage{algorithm}
\usepackage{algorithmic}

\usepackage{soul}
\usepackage{url}
\usepackage[hidelinks]{hyperref}
\usepackage[utf8]{inputenc}
\usepackage[small]{caption}
\usepackage{graphicx}
\usepackage{amsfonts}

\usepackage{amsthm}

\newtheorem{theorem}{Theorem}

\title{\LARGE \bf
Backward Imitation and Forward Reinforcement Learning via Bi-directional Model Rollouts
}

\author{Yuxin Pan$^{1}$ and Fangzhen Lin$^2$ 
\thanks{$^1$The author is with the Division of Emerging Interdisciplinary Areas under IPO, The Hong Kong University of Science and Technology, Clear Water Bay, Hong Kong. (e-mail: {\tt\small yuxin.pan@connect.ust.hk})} 
\thanks{$^2$The author is with the Department of Computer Science and Engineering, The Hong Kong University of Science and Technology, Clear Water Bay, Hong Kong. (e-mail: {\tt\small flin@cs.ust.hk})} 
}

\begin{document}

\maketitle
\thispagestyle{empty}
\pagestyle{empty}

\begin{abstract}

Traditional model-based reinforcement learning (RL) methods generate forward rollout traces using the learnt dynamics model to reduce interactions with the real environment. The recent model-based RL method considers the way to learn a backward model that specifies the conditional probability of the previous state given the previous action and the current state to additionally generate backward rollout trajectories. However, in this type of model-based method, the samples derived from backward rollouts and those from forward rollouts are simply aggregated together to optimize the policy via the model-free RL algorithm, which may decrease both the sample efficiency and the convergence rate. This is because such an approach ignores the fact that backward rollout traces are often generated starting from some high-value states and are certainly more instructive for the agent to improve the behavior. In this paper, we propose the backward imitation and forward reinforcement learning (BIFRL) framework where the agent treats backward rollout traces as expert demonstrations for the imitation of excellent behaviors, and then collects forward rollout transitions for policy reinforcement. Consequently, BIFRL empowers the agent to both reach to and explore from high-value states in a more efficient manner, and further reduces the real interactions, making it potentially more suitable for real-robot learning. Moreover, a value-regularized generative adversarial network is introduced to augment the valuable states which are infrequently received by the agent. Theoretically, we provide the condition where BIFRL is superior to the baseline methods. Experimentally, we demonstrate that BIFRL acquires the better sample efficiency and produces the competitive asymptotic performance on various MuJoCo locomotion tasks compared against state-of-the-art model-based methods.

\end{abstract}

\section{INTRODUCTION}

In recent years, with the tremendous advances of deep learning, the model-free reinforcement learning (RL) that utilizes the neural network to represent the policy or the value function has made remarkable progress~\cite{mnih2015human,schulman2015trust,haarnoja2018soft}, and can potentially automate various applications, such as autonomous driving and robotics manipulation. However, the model-free RL algorithm requires extensive interactions with the environment to obtain large amounts of data for policy optimization, which thus restricts it to mostly controllable simulated environments. In contrast, model-based methods have been found to be able to learn a near-optimal policy without enormous number of interactions with the environment~\cite{deisenroth2013survey}. This is because once a reliable model that is capable of precisely constructing the running mechanism of the environment is learnt, it can be used as a virtual environment to further drive the learning process, consequently improving the sample efficiency. But the downside of the model-based RL method is that it is often inferior to the model-free counterpart in terms of asymptotic performance.

Traditionally, in model-based RL methods, the agent interacts with the learnt transition function with uncertainty from the current state by executing a policy to generate the imaginary traces for policy optimization (called rollout). More recently, Goyal {\it et al.} \cite{goyal2018recall} proposed an impressive backward model that estimates the conditional probability of the previous state given the previous action and the current state to recall backward rollout traces, and showed that in many domains, especially those with sparse high-value states, this approach can acquire the improvement in terms of sample efficiency. Lai {\it et al.} \cite{lai2020bidirectional} then proposed a so-called bi-directional model-based policy optimization (BMPO) algorithm that uses both the traditional forward model (transition function) and the backward model to generate bi-directional rollout samples that reach into both the past and the future. However, such a model-based method simply aggregates the bi-directional model-generated data together to reinforce the policy via the model-free RL algorithm, which may cause the low sample efficiency and the slow convergence rate. This is because it is ignored that the traces derived from backward rollouts typically start from some high-value states and consequently can supervise the learning of agent like the expert demonstrations. To put it another way, the agent can efficiently learn to reach the high-value states through directly imitating the traces from backward rollouts, and then explore from those reached states for the future novel experiences to reinforce the policy. We thus believe that the bi-directional model rollout samples should be treated differently for the efficient policy learning.

To this end, we propose the backward imitation and forward reinforcement learning (BIFRL) framework where backward rollout traces generated from the presumed valuable states are used as expert demonstrations for the agent imitation of excellent behaviors, and then the interactions with the forward model are conducted to produce forward rollout traces for policy reinforcement. Specifically, BIFRL requires to generate high-value states as the starting points for the bi-directional model rollouts. From these states, backward rollouts are performed to generate the traces back to some initial states. These traces are then used like those from expert demonstrations to train the agent using imitation learning. Afterward, BIFRL conducts forward rollouts from these states and the resulting samples are utilized to further reinforce the policy using the RL algorithm. As a result, BIFRL empowers the agent to both reach to and explore from the valuable states to accelerate its policy learning and reduce the extensive interactions. Thus it has the potential to limit the damage to the robot and its environment if being applied in the real word, making it more applicable for real-robot learning. Additionally, the valuable states are indeed rare, and should be leveraged effectively and efficiently. Thus, a value-regularized generative adversarial network (GAN) is developed to discover more those particular useful states as state augmentation without the need of state labelling. Furthermore, we theoretically provide the condition where BIFRL shows the superiority against the baselines. Experimentally, we demonstrate that compared with previous state-of-the-art model-based models, BIFRL achieves higher sample efficiency and produces competitive asymptotic performance on various MuJoCo locomotion benchmark tasks~\cite{todorov2012mujoco}. 

\section{Related Work}

Compared with model-free RL, model-based RL is often more sample efficient given that it can generate simulated samples using the learned dynamics model. However, this often comes in the expense of the asymptotic performance as the learnt model may not be accurate. Previous works have tried to mitigate the compounding errors coming from the multi-step model rollouts by tuning the ratio of real to model-generated samples~\cite{mishra2017prediction,whitney2018understanding,wu2019model}. Nguyen \emph{et al.}~\cite{Nguyen2018ImprovingMR} and Xiao \emph{et al.}~\cite{xiao2019learning} considered different ways to estimate the cumulative model error by which the adaptive model-based planning horizon can be determined. MBPO~\cite{janner2019trust} generates truncated model rollouts branched from real states to cripple the influence of model error and provides the condition for the return improvement in the true dynamics. BMPO~\cite{lai2020bidirectional} combines the backward model and the forward model for bi-directional branched rollouts, which performs better asymptotically in both theory and practice. Our approach also utilizes the branched model rollout strategy to limit the compounding error as much as possible.

Several models try to integrate the backward model with RL algorithms. Goyal \emph{et al.}~\cite{goyal2018recall} and Edwards \emph{et al.}~\cite{edwards2018forward} used an imitation learning approach to additionally train the policy on the samples from the backward model. However, these methods are more like model-free RL and do not utilize the forward model to reduce the sample complexity. BMPO~\cite{lai2020bidirectional} uses both the backward and forward models for bi-directional model rollouts. Again, this approach treats the samples from backward rollouts and those from forward rollouts the same, resulting in impairing the efficiency.

Imitation learning is able to make RL model jump start to avoid the so-called cold-start~\cite{liang2018cirl,pan2020navigation}. The idea is that the experts will first demonstrate how to achieve a task, and from which an initial policy will be obtained using imitation learning. This assumes that there are experts who can perform the recordable and understandable demonstrations for the agent. Our use of imitation learning on the backward rollout data is inspired by this approach. In our approach, the backward model undertakes the role of experts, and backward rollouts directly generate the annotated data.

The high-value state can be defined by the estimated expected return or the received immediate reward under the current policy. But in most cases, it is infrequently encountered, and is definitely significant to launch model rollouts. BMPO~\cite{lai2020bidirectional} employs previously experienced states which are prioritized according to the state values. Goyal \emph{et al.}~\cite{goyal2018recall} relied on the learned Goal GAN~\cite{florensa2018automatic} to produce the starting states for backward rollouts, which however requires the annotations of high-value states as training data. By comparison, in our framework, a value-regularized GAN is introduced for state augmentation without the need of annotation, which forces the generator to directly output the critic-recognised valuable states.

\section{Preliminaries}

We formalize the RL problem using Markov Decision Process (MDP) that is described by the tuple $(\mathcal{S}, \mathcal{A}, \mathrm{T}, r, \gamma, \rho_{0})$. At each time step $t$, the agent in the state $s_t\in{\mathcal{S}}$ executes the action $a_t\in{\mathcal{A}}$, receives the reward $r_t=r(s_t, a_t)$, and moves to the next state $s_{t+1}$ according to the unknown transition function $\mathrm{T}: S \times A \to S$. $\rho_{0}$ denotes the initial state distribution. The goal of reinforcement learning is to find an optimal policy $\pi^{*}(a_t|s_t)$ that maximizes the expected sum of discounted rewards (expected return) as:
\begin{equation}
\pi^{*} = \mathop{\arg\max}_{\pi}\eta[\pi] = \mathop{\arg\max}_{\pi} \mathbb{E}_{\pi}[\sum\limits_{t=0}^{\infty}\gamma^t r(s_t, a_t)],
\end{equation}
where $\gamma \in [0, 1]$ is the discounted factor.

The transition function is assumed to be unknown in the model-free RL. In contrast, the Dyna-style algorithm~\cite{sutton1991dyna} that is one branch of the model-based paradigms learns a dynamics model (forward model) $p_{\theta}(s_{t+1}|s_{t}, a_{t})$ parameterized by $\theta$ as the true dynamics from agent interactions with the environment by conducting the policy $\pi_{\phi}(a_t|s_t)$ parameterized by $\phi$. In addition, in BIFRL, a backward model $\tilde{p}_{\theta^{\prime}}(s_{t-1}|s_{t},a_{t-1})$ and a backward policy $\tilde{\pi}_{\phi^{\prime}}(a_{t-1}|s_{t})$ are constructed for backward rollouts, which are parameterized by $\theta^{\prime}$ and $\phi^{\prime}$ respectively.

\begin{figure*}[t]
	\centering
	\includegraphics[width=1\textwidth]{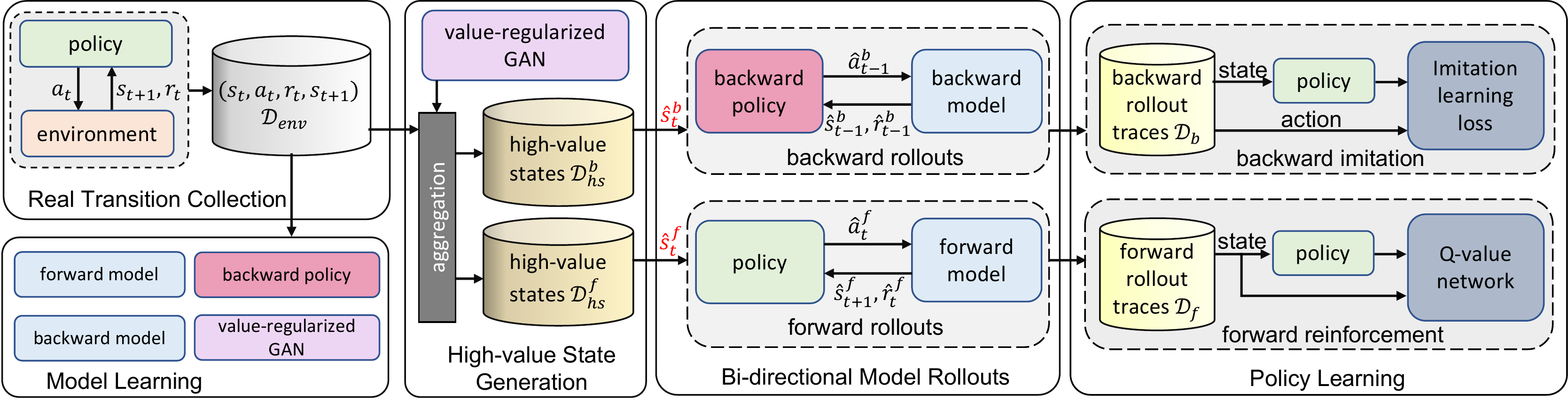}
	\caption{The overall architecture of BIFRL framework. The real transitions in the replay buffer $\mathcal{D}_{env}$ from agent interactions with the environment are collected for the model learning and the high-value state generation. In addition, the high-value state buffers are augmented by the samples from the learnt value-regularized GAN. Starting from the sampled valuable states, the bi-directional model rollouts are executed to collect the diverse simulated data for backward imitation of behavior and forward reinforcement of policy.}
	\label{fig1}
\end{figure*}

\section{Method}

In this section, we detail BIFRL algorithm that leverages the simulated samples from backward model rollouts in a different way with those generated from interactions with the forward model by performing the policy. Consequently, the agent is able to imitate excellent behaviors from backward rollout traces to reach the high-value states and reinforce its policy through explorations starting from those high-value states. Besides, BIFRL is built on the Model-based Policy Optimization (MBPO)~\cite{janner2019trust} algorithm that is one of state-of-the-art Dyna-style model-based methods and investigates forward truncated model rollouts. 

\subsection{BIFRL Framework}
Figure~\ref{fig1} illustrates the overview of BIFRL algorithm that consists of five indispensable and interdependent stages. The algorithm actually iterates between those stages to progressively improve the policy. In the real transition collection stage, the agent collects the data for the replay buffer $\mathcal{D}_{env}$ through interactions with the environment. In the model learning stage, the models including forward model, backward model, backward policy and value-regularized GAN are trained using the data in $\mathcal{D}_{env}$. In the high-value state generation stage, the states in $\mathcal{D}_{env}$ and those sampled from the value-regularized GAN are aggregated together, and then are prioritized and sampled for the backward imitation high-value state buffer $\mathcal{D}_{hs}^{b}$ and the forward reinforcement high-value state buffer $\mathcal{D}_{hs}^{f}$. In the bi-directional model rollouts stage, backward rollouts are performed starting from the states $\hat{s}^{b}_{t}$ randomly chosen from $\mathcal{D}_{hs}^{b}$ to collect the samples for the backward rollout data buffer $\mathcal{D}_{b}$. Meanwhile, forward rollouts start from the states $\hat{s}^{f}_{t}$ chosen from $\mathcal{D}_{hs}^{f}$ to gather the rollout transitions which are stored in the forward rollout data buffer $\mathcal{D}_{f}$. In the policy learning stage, the state-action pairs in $\mathcal{D}_{b}$ are regarded as expert demonstrations for the agent to imitate excellent behaviors, and the transitions from $\mathcal{D}_{f}$ are utilized to reinforce the policy via RL algorithm. Thus, the algorithm runs those five processes in cycles leading to the coupling between them. The BIFRL algorithm is presented in Algorithm~\ref{algorithm}. 

\textbf{Bi-directional Dynamics Models.} The architecture choices of bi-directional dynamics models are crucial for the policy optimization on the model-generated data, since the compounding model error can severely degenerate the multi-step rollouts. Therefore, in BIFRL the ensemble of bootstrapped probabilistic dynamics model~\cite{chua2018deep} is adopted for both the forward and backward models to capture two kinds of uncertainty. The one called aleatoric uncertainty arising from the inherent stochasticity of a system can be captured by each probabilistic model, and the other called epistemic uncertainty caused by the lack of sufficient training data can be remedied by the ensemble of models.  

To be specific, each of the forward model ensembles $\{p^{i}_{\theta}\}_{i=1}^{B}$ is represented by a neural network, and we use $\theta$ to denote the parameters of forward model ensembles. Each individual model takes the state $s_t$ and the action $a_t$ as inputs, and outputs the parameters (\emph{i.e.} means $\mu$ and diagonal covariances $\Sigma$) of the Gaussian distribution of the state variation $\Delta s_{t}=s_{t+1}-s_{t}$, denoted as:
\begin{equation}
    p_{\theta}^{i}(\Delta s_{t}|s_t,a_t)=\mathcal{N}(\mu_{\theta}^{i}(s_t,a_t), \Sigma_{\theta}^{i}(s_t,a_t)).
\end{equation}
Likewise, the same reparameterization technique is put to use for the trustworthy backward model, where each individual neural network with $\theta^{\prime}$ of the model ensembles yields the parameters of a Gaussian distribution written as:
\begin{equation}
    \tilde{p}_{\theta^{\prime}}^{i}(\Delta s_{t-1}|s_t,a_{t-1})=\mathcal{N}(\mu_{\theta^{\prime}}^{i}(s_t,a_{t-1}), \Sigma_{\theta^{\prime}}^{i}(s_t,a_{t-1})). 
\end{equation}
Note that these models can estimate the reward distribution as well, which is omitted for brevity.

\textbf{Backward Policy.} Relying on the learnt backward model, the previous state can be estimated given the current state and the previous action, which arouses a confusion about how to choose the previous actions for backward model rollouts. Thus, analogous to the setting of forward rollouts where the actions are naturally taken by the policy $\pi_{\phi}(a_t|s_t)$, we need the backward policy to predict the action $a_{t-1}$ given the state $s_t$. Since the multiple backward traces are required to be simultaneously generated from one starting state, the powerful probabilistic neural network with parameter $\phi^{\prime}$ is employed to output the parameters of the Gaussian distribution of the action $a_{t-1}$:
\begin{equation}
    \tilde{\pi}_{\phi^{\prime}}(a_{t-1}| s_t)=\mathcal{N}(\mu_{\phi^{\prime}}(s_t), \Sigma_{\phi^{\prime}}(s_t))
\end{equation}

\textbf{Value-regularized GAN.} In order to increase the quantity of high-value states, we propose a generative model called value-regularized GAN, which aims to generate the states that can be criticized as the high-value ones under the current policy by the critic and cannot move out of the real-state distribution. Unlike prior works~\cite{florensa2018automatic, goyal2018recall} that rely on the state labelling to train the Goal GAN, the value-regularized GAN is strengthened by the value regularization by which the generator $G(\cdot)$ treats maximum state value computed by the critic as another optimization objective besides the one defined by the discriminator $D(\cdot)$, resulting in removing the need for annotation. Additionally, due to the initially unstable learning of Q-value network $Q^{\pi}(\cdot)$, we consider the way to enable the generator to produce the high-reward states under the latest policy, which owes to the learnt dynamics model $p_{\theta}(\cdot)$ that can make reliable one-step prediction of the reward $r$. The value regularization can be described as:
\begin{equation}
    \mathcal{L}_{v} = \alpha Q^{\pi}(G(z), \pi_{\phi}(G(z))) + (1-\alpha)r
\end{equation}
\begin{equation}
    r \sim \mathcal{N}(\mu_{\theta}(G(z), \pi_{\phi}(G(z))), \Sigma_{\theta}(G(z), \pi_{\phi}(G(z))))
\end{equation}
where $\alpha \in (0, 1]$ is the weighting factor gradually increasing with the increasing number of training steps. The architecture of the value-regularized GAN is shown in Figure~\ref{fig2}.

\textbf{State Sampling Strategies.} Since the state diversity benefits the exploration, we adopt to collect the states stored in $D_{env}$ and those generated from value-regularized GAN together, instead of directly using the GAN-generated data, so as to sample the high-value states for model rollouts. We design different applicable state sampling strategies for the diverse model rollouts, because of the different training schemes on those rollout data. In detail, for the backward rollouts, the state value is chosen as the criterion to measure the importance of each state. Generally, the value of each state can be accurately computed by summing up the rewards from the current state $s_t$ to the termination state $s_T$ along the trace ($\sum_{t^{\prime}=t}^{T}  r(s_{t^{\prime}}, \pi_{\theta}(s_{t^{\prime}})) $), which however is not computationally efficient. Thus, the learnt value function $V^{\pi}(\cdot)$ is used to estimate the value, which is accessible and effective. Thus, BIFRL greedily picks the top $\mathcal{K}$ percent of the prioritized states and stores them in $\mathcal{D}_{hs}^{b}$.

Due to the greedy prioritization that focuses on a small subset of the states, the temporal difference (TD) errors on those low-priority states shrink slowly when using function approximation, which makes the system prone to over-fitting. Therefore, for the forward rollouts, the stochastic sampling method is introduced to ensure the probability of being sampled is monotonic in a state's priority and guarantee a non-zero probability for the low-value state. To further alleviate the lack of diversity, we treat the absolute TD error $\delta(\cdot)$ of the transition that is from the state to be prioritized to the next state under the current policy as another criterion. Concretely, we use the Boltzmann distribution to model the probability of sampling state $s^j$ defined as: 
\begin{equation}\label{eq7}
    p(s^j) = \frac{e^{\beta V^{\pi}(s^j) + (1-\beta) \delta(s^j)}}{\sum_{k=1}^{n}e^{\beta V^{\pi}(s^k) + (1-\beta) \delta(s^k)}},
\end{equation}
\begin{equation}
    \delta(s^j) = |V^{\pi}(s^j) - (r^{j} + \gamma V^{\pi}(\Delta s^{j} + s^{j})|,
\end{equation}
\begin{equation}
\Delta s^{j},r^{j} \sim \mathcal{N}(\mu_{\theta}(s^j,\pi_{\phi}(s^j)), \Sigma_{\theta}(s^j,\pi_{\phi}(s^j)))
\end{equation}
where $j$ is the index of state in the state collection with size $n$, and $\beta \in (0, 1)$ is to balance those two criterions. The states sampled in this way are stored in $\mathcal{D}_{hs}^{f}$.
\begin{figure}[t]
	\centering
	\includegraphics[width=1\columnwidth]{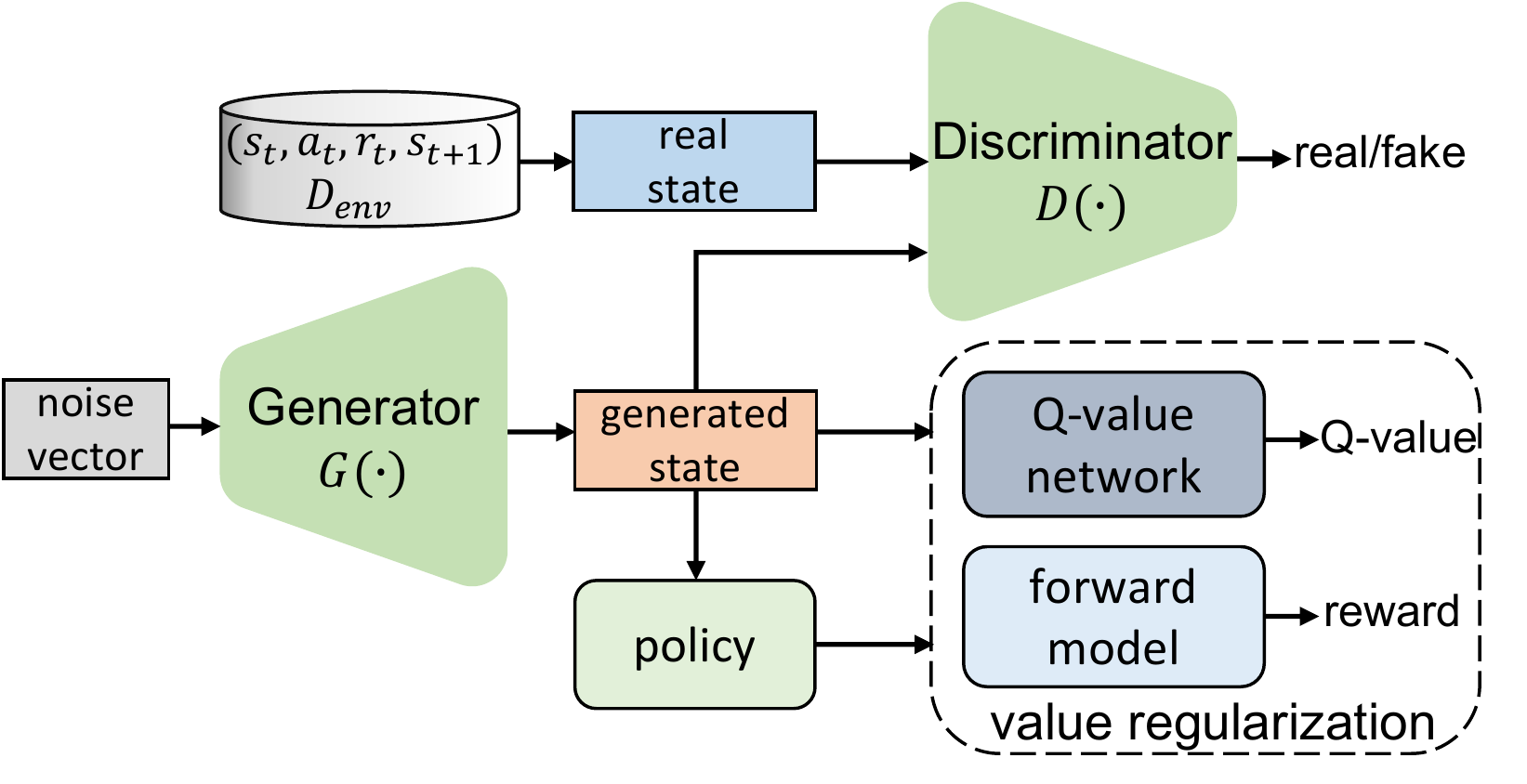}
	\caption{The architecture of value-regularized GAN.}
	\label{fig2}
\end{figure}

\subsection{Training}
In this part, we dedicate to introducing how to train the models in the model learning stage, and how to perform the policy learning to increasingly improve the agent's behavior.

\textbf{Bi-directional Dynamics Models.} In BIFRL, we choose the ensemble of bootstrapped probabilistic networks as the bi-directional dynamics models. In the training phase, each individual network of the ensembles is trained with different initializations and samples from the real transitions via maximizing likelihood. The corresponding loss function of the forward model can be written as:
\begin{equation}\label{eq11}
\begin{aligned}
	\mathcal{L}_{f}(\theta) = &- \sum\limits_{t=1}^{N}\log p_{\theta} (\Delta s_{t}| s_t, a_t) \\
	= &\sum\limits_{t=1}^{N}[\mu_{\theta}(s_t, a_t) - \Delta s_{t}]^{\top}\Sigma_{\theta}^{-1}(s_t, a_t) \\
	&[\mu_{\theta}(s_t, a_t) - \Delta s_{t}] + \log\det\Sigma_{\theta}(s_t, a_t)
\end{aligned}
\end{equation}
where $N$ is the total number of samples in $\mathcal{D}_{env}$. The loss function of the backward model can be further derived as:
\begin{equation}\label{eq12}
\begin{aligned}
	\mathcal{L}_{b}(\theta^{\prime}) = &- \sum\limits_{t=1}^{N} \log \tilde{p}_{\theta^{\prime}}(\Delta s_{t-1}|s_t, a_{t-1}) \\ = &\sum\limits_{t=1}^{N}[\mu_{\theta^{\prime}}(s_t, a_{t-1}) - \Delta s_{t-1}]^{\top}\Sigma_{\theta^{\prime}}^{-1}(s_t, a_{t-1}) \\
	&[\mu_{\theta^{\prime}}(s_t, a_{t-1}) - \Delta s_{t-1}] + \log\det\Sigma_{\theta^{\prime}}(s_t, a_{t-1})
\end{aligned}
\end{equation}

\textbf{Backward Policy.} The backward policy aims to make the backward rollouts resemble the real traces sampled by the current policy, which entails that it should be trained on the real data $(s_{t}, a_{t-1})$ by maximum likelihood estimation. Thus, the corresponding loss function is defined as:
\begin{equation}\label{eq13}
\begin{aligned}
    \mathcal{L}_{bp}(\phi^{\prime}) = & - \sum\limits_{t=1}^{N} \log \tilde{\pi}_{\phi^{\prime}} (a_{t-1}|s_t) \\
    = &\sum_{t=1}^{N} [\mu_{\phi^{\prime}}(s_{t})-a_{t-1}]^{\top}\Sigma_{\phi^{\prime}}^{-1}(s_{t}) \\
    &[\mu_{\phi^{\prime}}(s_{t})-a_{t-1}] + \log \det \Sigma_{\phi^{\prime}}(s_{t})
\end{aligned}
\end{equation}

\textbf{Value-regularized GAN.} In our proposed GAN, the generator $G(\cdot)$ parameterized by $\theta^{G}$ is trained to output the states which not only follow the real-state distribution in $\mathcal{D}_{env}$, but also are the critic-recognised valuable ones. The discriminator $D(\cdot)$ parameterized by $\theta^{D}$ is trained to distinguish the states in $\mathcal{D}_{env}$ from those being out of the presumed real-state distribution. The value regularized GAN is implemented on the basis of LSGAN~\cite{mao2017least}. The loss functions of generator and discriminator can be written as:
\begin{equation}\label{eq14}
    \mathcal{L}_{G}(\theta^{G}) = \mathbb{E}_{z\sim p_{z}(z)}[(D(G(z))-1)^2] - \lambda\mathcal{L}_{v}
\end{equation}
\begin{equation}\label{eq15}
    \mathcal{L}_{D}(\theta^{D}) = \mathbb{E}_{s\sim p_{s}(s)}[(D(s)-1)^2] + \mathbb{E}_{z\sim p_{z}(z)}[D(G(z))^2]
\end{equation}
where $p_{z}(z)$ is the standard normal distribution, $s$ is the real state following the real-state distribution $p_{s}(s)$, and $\lambda$ is the hyperparameter.

\textbf{Policy Learning.} The backward imitation of excellent behaviors and the forward reinforcement of policy are severally carried out on the data from $\mathcal{D}_{b}$ and $\mathcal{D}_{f}$. In detail, to imitate excellent behaviors from the traces leading to the high-value states, the policy is learnt on the simulated data in $\mathcal{D}_{b}$ via minimizing the loss function:
\begin{equation}\label{eq16}
	\mathcal{L_{I}}(\phi) = -\sum_{t=1}^{M}\log \pi_\phi(\hat{a}_{t}^{b}|\hat{s}_{t}^{b}),
\end{equation}
where the tuple $(\hat{s}_{t}^{b}, \hat{a}_{t}^{b})$ is sampled from $\mathcal{D}_{b}$ with size $M$.

After that, the agent should explore from the valuable states to further reinforce the policy. Thus, we optimize the policy on the samples in $\mathcal{D}_{f}$ through Soft Actor-Critic (SAC)~\cite{haarnoja2018soft} by minimizing the expected KL-divergence:
\begin{equation}\label{eq17}
    \mathcal{L}_{R}(\phi) = \mathbb{E}_{\hat{s}_{t}\sim D_{f}} [D_{KL}(\pi_{\phi} || \mathrm{exp}(Q^{\pi}-V^{\pi}))],
\end{equation}
where Q-value network $Q^{\pi}(\cdot)$ and state value network $V^{\pi}(\cdot)$ are trained via minimizing the variants of TD error. Please refer to~\cite{haarnoja2018soft} for the details of SAC algorithm.

\section{Theoretical Analysis}
\label{theoretical_analysis}
In this section, we theoretically provide the condition where BIFRL algorithm can produce the better asymptotic performance than previous methods, which is also instructive for hyperparameter settings in experiments. On the one hand, prior works~\cite{janner2019trust, lai2020bidirectional} have studied the discrepancy between the expected return in the actual environment and that in the branched rollouts, which is applicable for the analysis on the forward rollout return in BIFRL. On the other hand, unlike those works, BIFRL views backward rollout traces $\tau^{b}=[\hat{s}_{T}^{b}, \hat{a}_{T-1}^{b}, ..., \hat{a}_t^{b}, \hat{s}_t^{b}]$ as expert demonstrations $\tau^{e}=[\hat{s}_t^{b}, \hat{a}_t^{b}, ..., \hat{a}_{T-1}^{b}, \hat{s}_{T}^{b}]$ to train the policy, and assumes those expert samples are sampled from the interactions with a presumed model $p^{e}(s_{t+1}|s_{t}, a_{t})$ by executing a presumed expert policy $\pi^{e}(a_{t}|s_{t})$. We thus investigate the discrepancy between the return from the current policy $\pi$ interactions with the real dynamics $p$, denoted as $\eta[\pi]$, and the return from the expert policy $\pi^{e}$ interactions with the presumed model $p^{e}$, denoted as $\eta[\pi^{e}]$ . Let $k_b$ be the length of backward rollouts, $k_f$ be length of forward rollouts.

\begin{theorem} 
\label{theorem1}
We assume that the expected total variation distance between $\pi$ and $\pi^{e}$ at each timestep t is bounded as $\epsilon_{\pi} = \max_{t}E_{(s,a) \sim \pi_{t}, p}[D_{TV}(\pi_{t}^{e}(a|s)||\pi_{t}(a|s))]$, and the expected total variation distance between $p$ and $p^{e}$ at each time step t is bounded as $\epsilon_{m} = \max_{t}E_{(s,a,s^{\prime}) \sim \pi_{t}, p^{e}}[D_{TV}(p^{e}(s^{\prime}|s,a)||p(s^{\prime}|s,a))]$. Let $r_{\mathrm{max}}$ denote the maximum reward. Then the upper bound of the discrepancy between $\eta[\pi]$ and $\eta[\pi^{e}]$ is:
\begin{equation}
    |\eta[\pi] - \eta[\pi^{e}]| \leq \frac{2r_{\mathrm{max}}(1-\gamma^{k_b+1})}{1-\gamma}(\epsilon_{\pi}+k_b(\epsilon_{\pi}+\epsilon_{m}))
\end{equation}
\end{theorem}

\begin{proof}
Please refer to Appendix for the proof.
\end{proof}

We observe that BMPO~\cite{lai2020bidirectional} ignores the variation distance between the forward policy and the backward policy (corresponding to $\epsilon_{\pi}$ in BIFRL), and consequently set $k_b=k_f$ to obtain the tighter return discrepancy bound, compared to MBPO~\cite{janner2019trust}. However, based on Theorem~\ref{theorem1}, $\epsilon_{\pi}$ cannot be eliminated in the return discrepancy. Thus, compared to forward rollouts that only have the model error, backward rollouts have both the model error and the policy divergence. In consequence, we propose the condition $k_b < k_f$ where BIFRL can perform better asymptotically. The experiments also support our viewpoints.

\section{Experiments}

The aims of our experimental evaluation are to answer two primary questions: 1) How does BIFRL algorithm perform on various benchmarks compared with state-of-the-art model-free and model-based methods? 2) what are the contributions of the critical components in BIFRL?

\subsection{Performance Comparisons}

In this section, we evaluate BIFRL and other methods on five MuJoCo continuous control tasks~\cite{todorov2012mujoco} where the agent requires to learn the moving forward gaits. The agents we used are Hopper, Walker2D and Ant. And we add two variants of Hopper and Walker2D without early termination states, denoted as Hopper-NT and Walker2D-NT~\cite{langlois2019benchmarking}. Additionally, we choose Pendulum task~\cite{brockman2016openai} to evaluate the performance of BIFRL on traditional control task. Please refer to~\cite{todorov2012mujoco,langlois2019benchmarking,brockman2016openai} for the environment settings.

\begin{figure*}[t]
	\centering
	\includegraphics[width=1\textwidth]{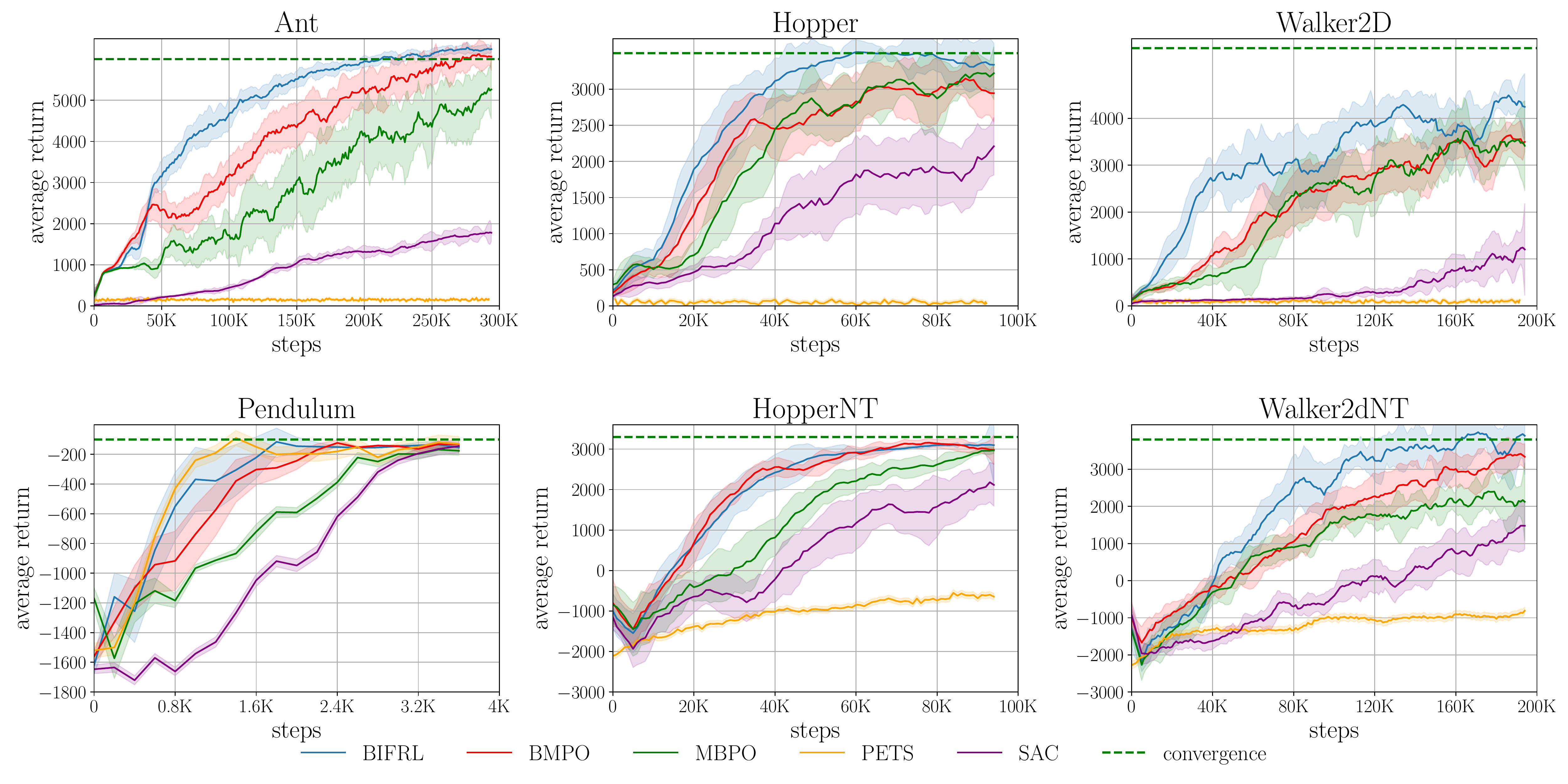}
	\caption{Learning curves of BIFRL and four baselines in different continuous control benchmarks. Solid lines indicate the mean and shaded areas depict the standard error of five trials with different random seeds in training. In each trial, all the algorithms are evaluated every 1000 environment steps (200 steps for Pendulum), and on each evaluation, the average return is computed over ten episodes. The dashed reference lines are the indications of asymptotic performance of SAC.}
	\label{main_expr}
\end{figure*}

We compare BIFRL algorithm with previous state-of-the-art algorithms in those six domains. To be specific, for model-based methods, we compare against MBPO~\cite{janner2019trust} that is the backbone model of our method, and BMPO~\cite{lai2020bidirectional} that treats the samples from backward rollouts in the same way as those from forward rollouts. We also compare to PETS~\cite{chua2018deep} that incorporates the learnt model ensembles into model predicative control (MPC)~\cite{camacho2013model} for planning. For model-free methods, we compare to SAC~\cite{haarnoja2018soft}, which is used for policy reinforcement in BIFRL. We do not compare with the backtracking model method~\cite{goyal2018recall} as it is not a purely model-based method and its improvements on both the performance and efficiency are limited compared with SAC. In addition, some previous model-based methods only involving the forward model rollouts (\emph{e.g.} MAAC~\cite{clavera2019model}, MoPAC~\cite{morgan2021model}) can be integrated with BIFRL. It is unfair to directly compare BIFRL against those model-based methods, and the compatibility of BIFRL with those baselines can be researched in future works. In experiments, we use the default hyperparameter settings in both the baselines and the backbone model of BIFRL for fair comparison. And the unique hyperparameters of BIFRL are included in Appendix.

The learning curves of all the methods are shown in Figure~\ref{main_expr}. We observe that in the four domains, including Ant, Hopper, Walker2D and Walker2D-NT, BIFRL acquires better sample efficiency resulting in the faster convergence rate, and meanwhile to some extent improves the asymptotic performance. In Hopper-NT, BIFRL does not make considerable improvements. This is probably because the high-value states cannot be frequently encountered in this domain. And in Pendulum, PETS is superior to all the other model-based methods. One possible reason is that MPC-based methods are more applicable for the traditional control tasks.

More qualitative results are shown in our supplementary materials, which provides the qualitative comparisons among BIFRL, BMPO, MBPO and SAC.

\begin{figure*}[t]
	\centering
	\subfigure[Ablations]{\label{ablations}\includegraphics[width=0.5\columnwidth]{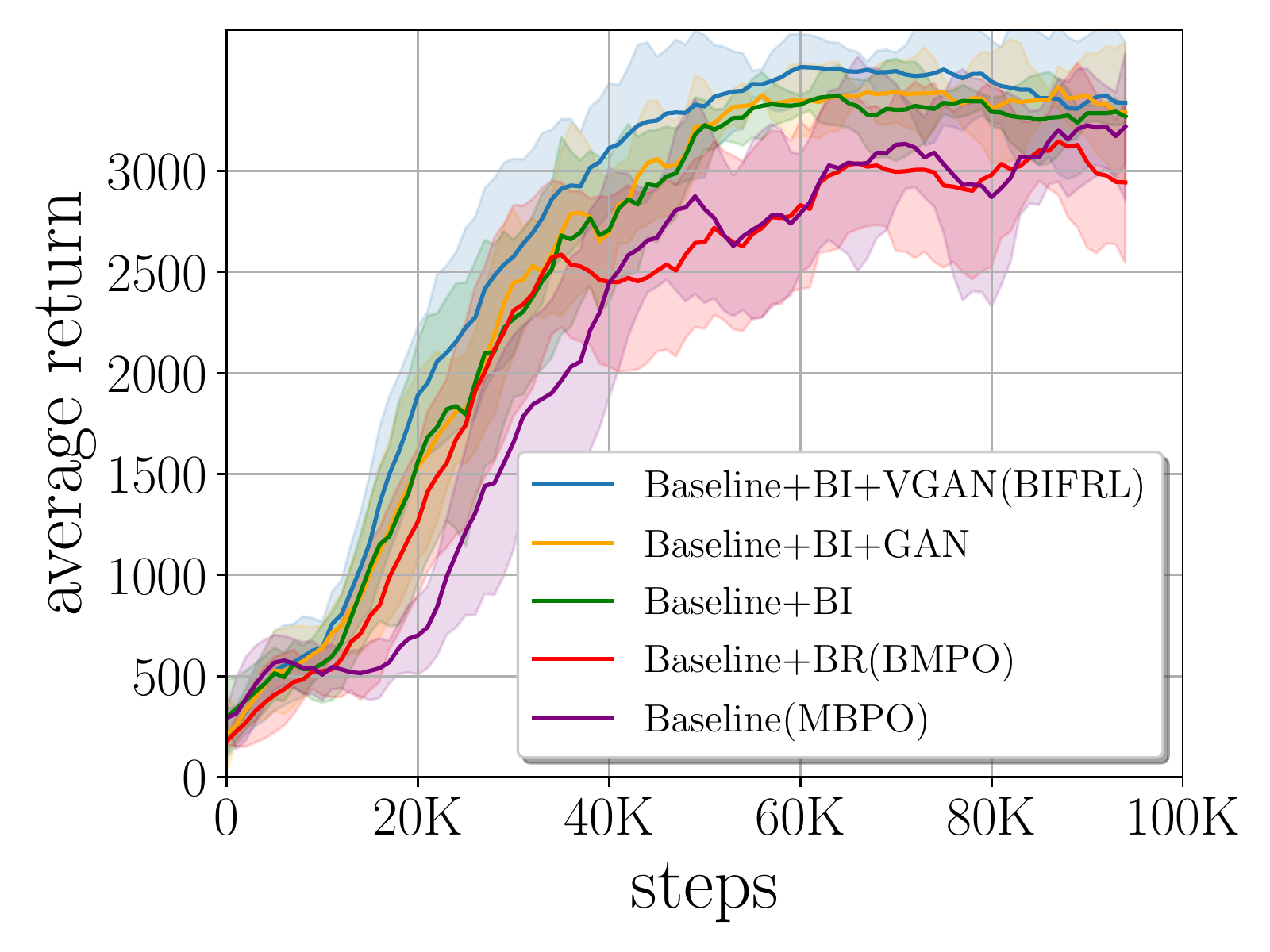}}
	\subfigure[Top $\mathcal{K}$ Percent of Samples ]{\label{top_samples}\includegraphics[width=0.5\columnwidth]{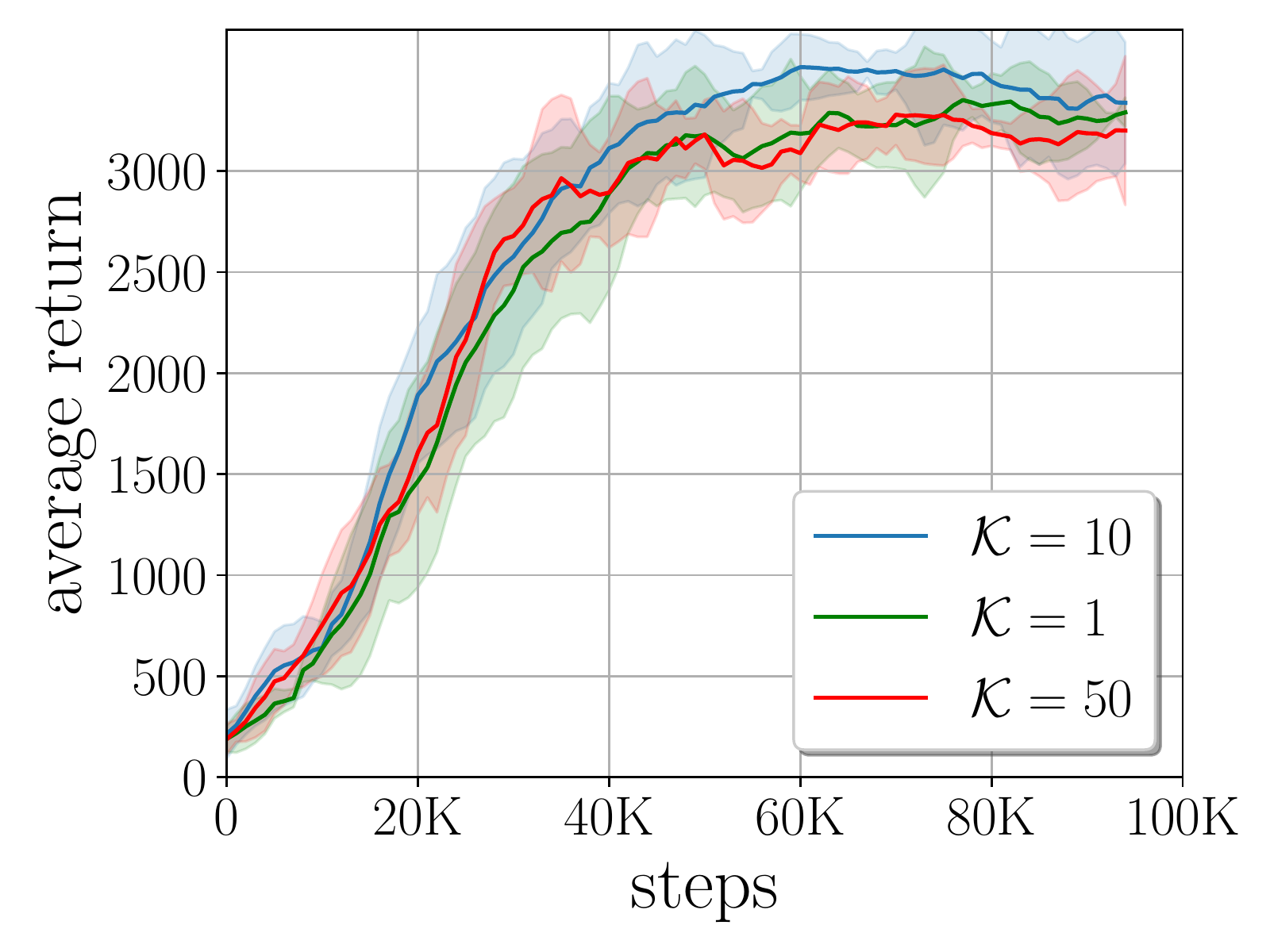}}
	\subfigure[Backward Length]{\label{backward_length}\includegraphics[width=0.5\columnwidth]{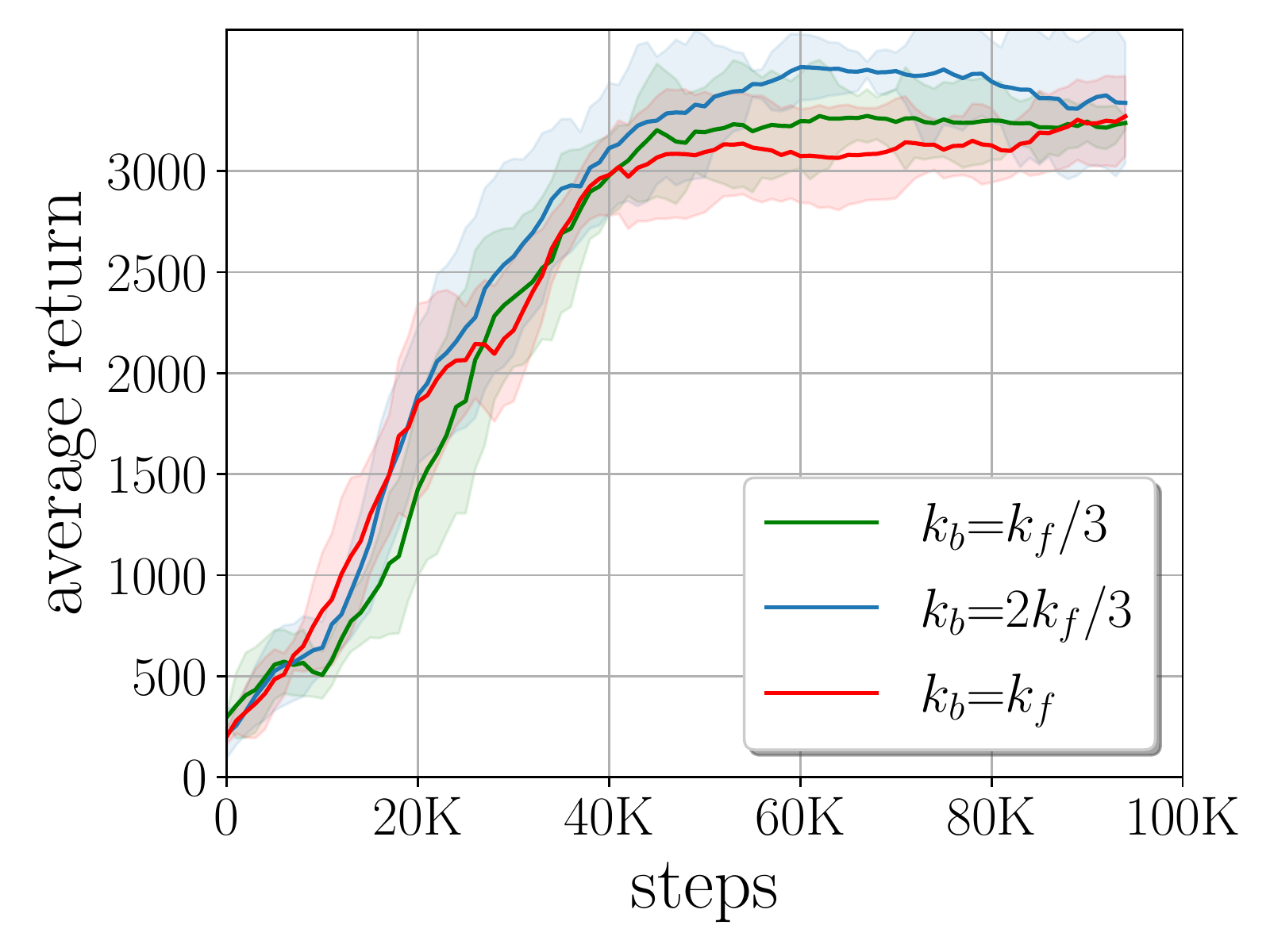}}
	\subfigure[Validation Loss]{\label{validation_loss}\includegraphics[width=0.5\columnwidth]{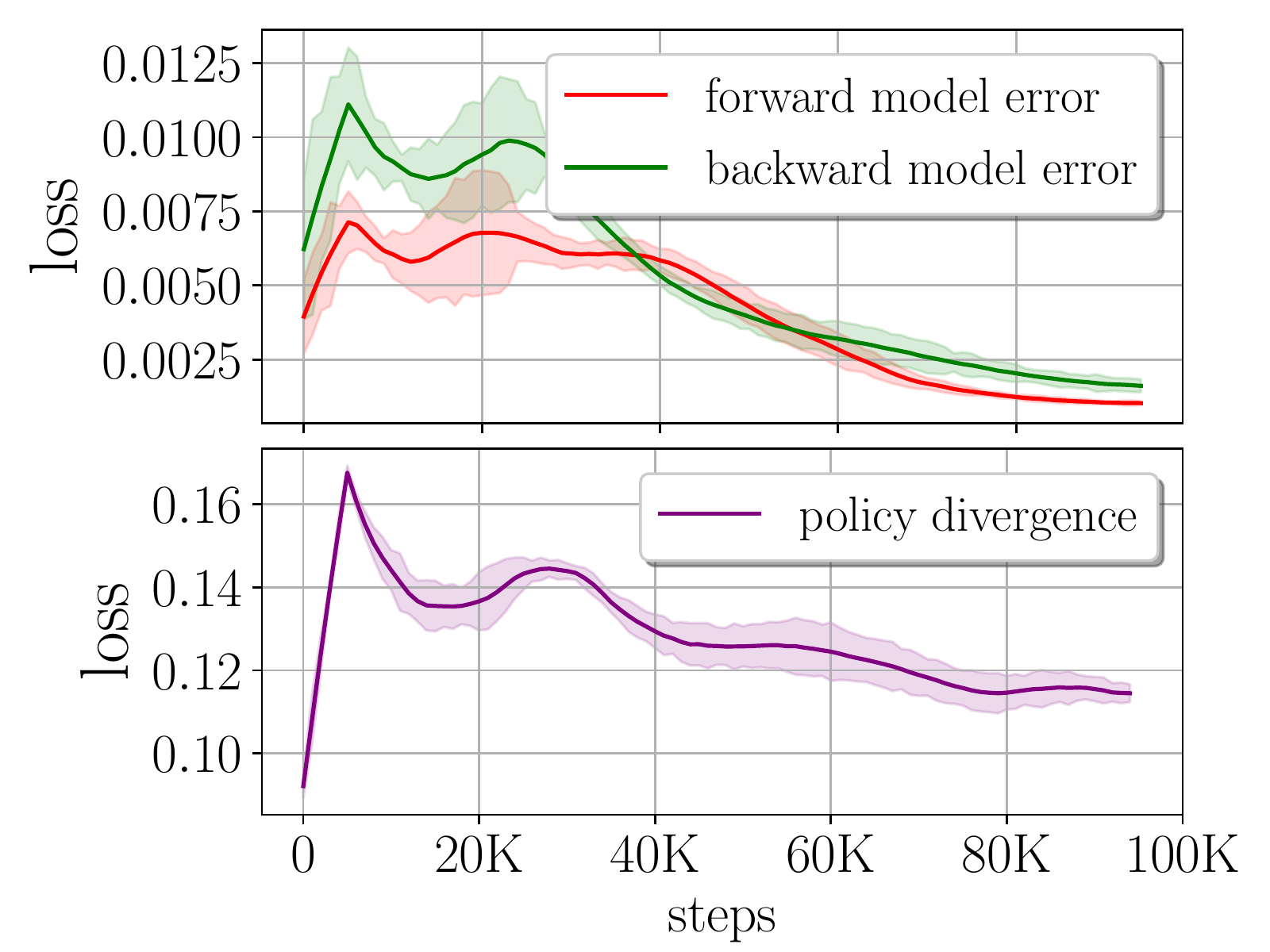}}
	\caption{Design evaluation of BIFRL on Hopper task. (a) Ablation study of two components: backward imitation (BI) and value-regularized GAN (VGAN). (b) The influence of $\mathcal{K}$ (the quantity of sampled high-value states for backward rollouts). (c) The influence of backward rollout length $k_b$ when fixing forward rollout length $k_f$. (d) The bi-directional model errors and policy divergence on validation set.}
	\label{extra_expr}
\end{figure*}

\subsection{Design Evaluation}

In this section, we investigate the contribution of each component and study the important hyperparameters of BIFRL on Hopper task.

\textbf{Ablation Study.}
As illustrated in Figure~\ref{ablations}, we conduct the ablation study to evaluate the contributions of two main components in BIFRL: 1) backward imitation that views backward rollout samples as expert demonstrations, denoted as BI; 2) value-regularized GAN that augments the high-value states, denoted as VGAN. To evaluate the importance of BI, we compare the performance among three models: 1) MBPO that optimizes the policy on forward rollout samples via RL algorithm (Baseline); 2) BMPO that adds the backward rollout samples to the baseline model (Baseline+BR); 3) the model that employs backward imitation (Baseline+BI). We observe that the baseline model with backward imitation convergences faster, which demonstrates that using backward rollout samples as expert demonstrations is more efficient than simply feeding those samples to RL algorithm. To characterize the importance of VGAN, we additionally choose two models: 1) using vanilla LSGAN~\cite{florensa2018automatic} to generate states (Baseline+BI+GAN); 2) using value-regularized GAN for high-value states (Baseline+BI+VGAN). We find that VGAN can contribute more to the performance in terms of convergence rate, which verifies the effectiveness of VGAN.

\textbf{Hyperparameter Study.}
We further investigate the important hyperparameters in our algorithm. Firstly, we conduct experiments with different $\mathcal{K}$ which indicates the quantity of sampled high-value states for backward rollouts. We set $\mathcal{K}=1, 10, 50$ respectively to study its influence. The results are shown in Figure~\ref{top_samples}. We find that the large number of sampled states can impair the performance. This probably because the larger $\mathcal{K}$ makes the high-value states sparse in $\mathcal{D}_{hs}^{b}$. We observe that too few high-value states also degrades the performance. Thus, it is crucial to choose an appropriate quantity of high-value states for backward rollouts.

The strategy of increasing rollout length linearly has shown the effectiveness in MBPO and BMPO. However, we require to study how to choose the backward rollout length $k_b$ in our algorithm when keeping the same forward rollout length $k_f$ as BMPO. In section~\ref{theoretical_analysis}, we derive the condition $k_b < k_f$ where BIFRL can perform better, we thus vary $k_b$ from 0 to $k_f$. As shown in Figure~\ref{backward_length}, setting $k_b=\frac{2}{3}k_f$ provides the better results, while neither too short or too long backward rollout length improves the performance. In practice, we use $k_b=\frac{2}{3}k_f$ in most domains except Walker2D and Walker2D-NT where $k_b$ is set to the same as $k_f$ due to $k_f=1$ in the original MBPO paper. Please refer to Appendix for the unique hyperparameters of BIFRL.

To prove that the policy divergence $\epsilon_{\pi}$ cannot be eliminated in Theorem~\ref{theorem1}, we choose the mean squared error (MSE) to measure both the bi-directional model errors and the policy divergence on validation set. Figure~\ref{validation_loss} shows the changes of these errors in the training phase. We notice that the policy divergence is two orders of magnitude larger than the bi-directional model errors at the end of training, which supports our view, \emph{i.e.} $\epsilon_{\pi}$ cannot be eliminated, and further experimentally verifies the validity of condition $k_b < k_f$.

\section{Conclusion}

In this paper, we propose a new model-based RL framework called backward imitation and forward reinforcement learning (BIFRL) framework where the agent views backward rollout traces as expert demonstrations to imitate the excellent behaviors from those traces, and uses forward rollout transitions to optimize the policy via RL algorithm. Theoretically, we provide the condition where BIFRL can produce the better asymptotic performance, and the experiment results actually support our view and derived condition in the analysis. Furthermore, BIFRL can empirically improves the sample efficiency and produces the competitive asymptotic performance on various locomotion tasks. For future work, we will investigate the compatibility of BIFRL with other model-based methods, and study its effectiveness in more real scenarios.

\section*{APPENDIX}

\subsection{Proof}
\begin{proof}
Firstly, the state-action space is assumed to be finite due to the restrictions on states or actions of the robot. According to the definition of the return in terms of the occupancy measure and the total variation distance, the return discrepancy bound can be derived as:
\begin{equation}
\begin{aligned}
    |\eta[\pi] - \eta[\pi^{e}]| &= |\sum_{s,a}\sum_{t=0}^{k_b}\gamma^{t}r(s,a)(p_{t}(s,a)-p_{t}^{e}(s,a))| \\
    &\leq 2r_{\mathrm{max}}\sum_{s,a}\sum_{t=0}^{k_b}\gamma^{t}\frac{1}{2}|p_{t}(s,a)-p_{t}^{e}(s,a)|\\
    &\leq 2r_{\mathrm{max}} \sum_{t=0}^{k_b}\gamma^{t}D_{TV}(p_{t}(s,a)||p_{t}^{e}(s,a)),
\end{aligned}
\end{equation}

Next, according to Lemma B.1 in MBPO~\cite{janner2019trust}, we convert joint distribution to marginal distribution, thus we have:
\begin{equation}
\begin{split}
    D_{TV}(p_{t}(s,a)||p_{t}^{e}(s,a)) \leq D_{TV}(p_{t}(s)||p_{t}^{e}(s)) + \\ \max_{t}E_{s\sim p_{t}(s)}[D_{TV}(\pi_{t}(a|s)||\pi_{t}^{e}(a|s))] 
\end{split}
\end{equation}
Then, let $\xi_{t}=D_{TV}(p_{t}(s)||p_{t}^{e}(s))$, and inspired by Lemma B.1 in BMPO~\cite{lai2020bidirectional}, we have:
\begin{equation}
\begin{aligned}
    \xi_{t} \leq & E_{(s^{\prime},a)\sim p_{t+1}(s^{\prime},a)}[D_{TV}(p(s|s^{\prime},a)||p^{e}(s|s^{\prime},a))] \\
    & + D_{TV}(p_{t+1}(s^{\prime},a)||p_{t+1}^{e}(s^{\prime},a))
\end{aligned}
\end{equation}
Here, we make the assumption without loss of generality:
\begin{equation}
    \epsilon_{m} \geq \max_{t}E_{(s^{\prime}, a)\sim{p_{t}(s^{\prime},a)}}[D_{TV}(p(s|s^{\prime},a)||p^{e}(s|s^{\prime},a))]
\end{equation}
After that, we can iteratively do the above decomposition to obtain:
\begin{equation}
\begin{split}
    \xi_{t} &\leq (\epsilon_{m}+\epsilon_{\pi}) + \xi_{t+1} 
            \leq (\epsilon_{m}+\epsilon_{\pi})(k_{b}-t) + \xi_{k_{b}} \\
            &\leq k_{b}(\epsilon_{m}+\epsilon_{\pi})
\end{split}
\end{equation}
where $\xi_{k_{b}}=0$, because the real states sampled from $D_{hs}^{b}$ are used as the starting states for the backward rollouts.

Finally, the upper bound of discrepancy between $\eta[\pi]$ and $\eta[\pi^{e}]$ can be described as:
\begin{equation}
    |\eta[\pi] - \eta[\pi^{e}]| \leq \frac{2r_{\mathrm{max}}(1-\gamma^{k_b+1})}{1-\gamma}(\epsilon_{\pi}+k_b(\epsilon_{\pi}+\epsilon_{m}))
\end{equation}
\end{proof}

\begin{table}[h]
\centering
\caption{Hyperparameter settings in BIFRL.}
\begin{tabular}{c|c|c|c}
\toprule[1.0pt]
Environment & $k_b$ & $k_f$ & $e_{\alpha}$ \\
\midrule[0.5pt]
Pendulum & $1\to3~|~1\to5$ & $1\to5~|~1\to5$ & 10 \\
\midrule[0.5pt]
Hopper & $1\to10~|~20\to150$ & $1\to15~|~20\to150$ & 20 \\
\midrule[0.5pt]
Hopper-NT & $1\to10~|~20\to150$ & $1\to15~|~20\to150$ & 20 \\
\midrule[0.5pt]
Walker2D & 1 & 1 & 40 \\
\midrule[0.5pt]
Walker2D-NT & 1 & 1 & 40 \\
\midrule[0.5pt]
Ant & $1\to16~|~20\to100$ & $1\to25~|~20\to100$ & 40 \\
\bottomrule[1.0pt]
\end{tabular}
\label{table1}
\end{table}

\subsection{Hyperparameters}
Table~\ref{table1} provides the unique hyperparameters in BIFRL. $x \to y~|~a \to b$ indicates clipped linear function, \emph{i.e.} for epoch $e$, $f(e)=clip((x+\frac{e-a}{b-a}(x-y)), x, y)$. $\alpha$ increases from 0.2 to 0.95 with the epoch changing from 1 to $e_{\alpha}$. $\beta$ is set as 0.7 in all the domains. $\lambda$ is 0.0001 in most domains except Pendulum where it is 0.001. The ratio between states from $\mathcal{D}_{env}$ and those from VGAN is 1. Other hyperparameters not listed here are the same as those in MBPO~\cite{lai2020bidirectional}. Additionally, the network architectures and the experimental settings in value-regularized GAN are the same as those in Goal GAN~\cite{florensa2018automatic} that is also built on LSGAN~\cite{mao2017least}.

\begin{algorithm}[t]
	\caption{BIFRL Algorithm}
	\label{algorithm}
	\textbf{Initialization}: policy $\pi_{\phi}$, backward policy $\tilde{\pi}_{\phi^{\prime}}$, forward model $p_{\theta}$, backward model $\tilde{p}_{\theta^{\prime}}$, and VGAN (generator, discriminator)
	\begin{algorithmic}[1]
		\FOR{$N~epochs$}
		\STATE Collect data from environment; add to $\mathcal{D}_{env}$
		\STATE Train $p_{\theta}$, $\tilde{p}_{\theta^{\prime}}$, $\tilde{\pi}_{\phi^{\prime}}$ and VGAN using $\mathcal{D}_{env}$ via applying gradient descent on Equation~\ref{eq11},~\ref{eq12},~\ref{eq13},~\ref{eq14} and~\ref{eq15}
		\STATE Aggregate states from $\mathcal{D}_{env}$ and VGAN
		\STATE Sample top $\mathcal{K}$ percent of high-value states from aggregated states; add to $\mathcal{D}_{hs}^{b}$
		\STATE Sample high-value states from aggregated states according to Equation~\ref{eq7}; add to $\mathcal{D}_{hs}^{f}$
		\FOR{$M_{1}~backward~rollouts$}
		\STATE Perform $k_b$ steps backward rollouts; add to $\mathcal{D}_{b}$ 
		\ENDFOR
		\FOR{$G_{1}~gradient~updates$}
		\STATE Update policy on $\mathcal{D}_{b}$ by optimizing Equation~\ref{eq16}
		\ENDFOR
		\FOR{$M_{2}~forward~rollouts$}
		\STATE Perform $k_f$ steps forward rollouts; add to $\mathcal{D}_{f}$ 
		\ENDFOR
		\FOR{$G_{2}~gradient~updates$}
		\STATE Update policy on $\mathcal{D}_{f}$ by optimizing Equation~\ref{eq17}
		\ENDFOR
		\ENDFOR
	\end{algorithmic}
\end{algorithm}

\bibliographystyle{IEEEtran}

\begin{thebibliography}{10}
\providecommand{\url}[1]{#1}
\csname url@rmstyle\endcsname
\providecommand{\newblock}{\relax}
\providecommand{\bibinfo}[2]{#2}
\providecommand\BIBentrySTDinterwordspacing{\spaceskip=0pt\relax}
\providecommand\BIBentryALTinterwordstretchfactor{4}
\providecommand\BIBentryALTinterwordspacing{\spaceskip=\fontdimen2\font plus
\BIBentryALTinterwordstretchfactor\fontdimen3\font minus
  \fontdimen4\font\relax}
\providecommand\BIBforeignlanguage[2]{{%
\expandafter\ifx\csname l@#1\endcsname\relax
\typeout{** WARNING: IEEEtran.bst: No hyphenation pattern has been}%
\typeout{** loaded for the language `#1'. Using the pattern for}%
\typeout{** the default language instead.}%
\else
\language=\csname l@#1\endcsname
\fi
#2}}

\bibitem{mnih2015human}
V.~Mnih, K.~Kavukcuoglu, D.~Silver, A.~A. Rusu, J.~Veness, M.~G. Bellemare,
  A.~Graves, M.~Riedmiller, A.~K. Fidjeland, G.~Ostrovski, \emph{et~al.},
  ``Human-level control through deep reinforcement learning,'' \emph{nature},
  vol. 518, no. 7540, pp. 529--533, 2015.

\bibitem{haarnoja2018soft}
T.~Haarnoja, A.~Zhou, P.~Abbeel, and S.~Levine, ``Soft actor-critic: Off-policy
  maximum entropy deep reinforcement learning with a stochastic actor,'' in
  \emph{International Conference on Machine Learning}, 2018.

\bibitem{schulman2015trust}
J.~Schulman, S.~Levine, P.~Abbeel, M.~Jordan, and P.~Moritz, ``Trust region
  policy optimization,'' in \emph{International conference on machine
  learning}, 2015, pp. 1889--1897.

\bibitem{deisenroth2013survey}
M.~P. Deisenroth, G.~Neumann, and J.~Peters, \emph{A survey on policy search
  for robotics}.\hskip 1em plus 0.5em minus 0.4em\relax now publishers, 2013.

\bibitem{lai2020bidirectional}
H.~Lai, J.~Shen, W.~Zhang, and Y.~Yu, ``Bidirectional model-based policy
  optimization,'' in \emph{International Conference on Machine Learning}.\hskip
  1em plus 0.5em minus 0.4em\relax PMLR, 2020, pp. 5618--5627.

\bibitem{whitney2018understanding}
W.~Whitney and R.~Fergus, ``Understanding the asymptotic performance of
  model-based rl methods,'' 2018.

\bibitem{mishra2017prediction}
N.~Mishra, P.~Abbeel, and I.~Mordatch, ``Prediction and control with temporal
  segment models,'' in \emph{International Conference on Machine Learning},
  2017.

\bibitem{wu2019model}
Y.-H. Wu, T.-H. Fan, P.~J. Ramadge, and H.~Su, ``Model imitation for
  model-based reinforcement learning,'' in \emph{International Conference on
  Learning Representations}, 2020.

\bibitem{Nguyen2018ImprovingMR}
N.~M. Nguyen, ``Improving model-based rl with adaptive rollout using
  uncertainty estimation,'' 2018.

\bibitem{xiao2019learning}
C.~Xiao, Y.~Wu, C.~Ma, D.~Schuurmans, and M.~M{\"u}ller, ``Learning to combat
  compounding-error in model-based reinforcement learning,'' in \emph{NeurIPS
  2019 Deep Reinforcement Learning Workshop}, 2019.

\bibitem{janner2019trust}
M.~Janner, J.~Fu, M.~Zhang, and S.~Levine, ``When to trust your model:
  Model-based policy optimization,'' in \emph{Advances in Neural Information
  Processing Systems}, 2019, pp. 12\,519--12\,530.

\bibitem{edwards2018forward}
A.~D. Edwards, L.~Downs, and J.~C. Davidson, ``Forward-backward reinforcement
  learning,'' \emph{arXiv preprint arXiv:1803.10227}, 2018.

\bibitem{goyal2018recall}
A.~Goyal, P.~Brakel, W.~Fedus, S.~Singhal, T.~Lillicrap, S.~Levine,
  H.~Larochelle, and Y.~Bengio, ``Recall traces: Backtracking models for
  efficient reinforcement learning,'' in \emph{International Conference on
  Learning Representations}, 2019.

\bibitem{liang2018cirl}
X.~Liang, T.~Wang, L.~Yang, and E.~Xing, ``Cirl: Controllable imitative
  reinforcement learning for vision-based self-driving,'' in \emph{Proceedings
  of the European Conference on Computer Vision (ECCV)}, 2018, pp. 584--599.

\bibitem{pan2020navigation}
Y.~Pan, J.~Xue, P.~Zhang, W.~Ouyang, J.~Fang, and X.~Chen, ``Navigation command
  matching for vision-based autonomous driving,'' in \emph{2020 IEEE
  International Conference on Robotics and Automation}.\hskip 1em plus 0.5em
  minus 0.4em\relax IEEE, 2020, pp. 4343--4349.

\bibitem{camacho2013model}
E.~F. Camacho and C.~B. Alba, \emph{Model predictive control}.\hskip 1em plus
  0.5em minus 0.4em\relax Springer Science \& Business Media, 2013.

\bibitem{chua2018deep}
K.~Chua, R.~Calandra, R.~McAllister, and S.~Levine, ``Deep reinforcement
  learning in a handful of trials using probabilistic dynamics models,'' in
  \emph{Advances in Neural Information Processing Systems}, 2018, pp.
  4754--4765.

\bibitem{todorov2012mujoco}
E.~Todorov, T.~Erez, and Y.~Tassa, ``Mujoco: A physics engine for model-based
  control,'' in \emph{2012 IEEE/RSJ International Conference on Intelligent
  Robots and Systems}.\hskip 1em plus 0.5em minus 0.4em\relax IEEE, 2012, pp.
  5026--5033.

\bibitem{langlois2019benchmarking}
E.~Langlois, S.~Zhang, G.~Zhang, P.~Abbeel, and J.~Ba, ``Benchmarking
  model-based reinforcement learning,'' \emph{arXiv preprint arXiv:1907.02057},
  2019.

\bibitem{florensa2018automatic}
C.~Florensa, D.~Held, X.~Geng, and P.~Abbeel, ``Automatic goal generation for
  reinforcement learning agents,'' in \emph{International conference on machine
  learning}.\hskip 1em plus 0.5em minus 0.4em\relax PMLR, 2018, pp. 1515--1528.

\bibitem{sutton1991dyna}
R.~S. Sutton, ``Dyna, an integrated architecture for learning, planning, and
  reacting,'' \emph{ACM Sigart Bulletin}, vol.~2, no.~4, pp. 160--163, 1991.

\bibitem{mao2017least}
X.~Mao, Q.~Li, H.~Xie, R.~Y. Lau, Z.~Wang, and S.~Paul~Smolley, ``Least squares
  generative adversarial networks,'' in \emph{Proceedings of the IEEE
  international conference on computer vision}, 2017, pp. 2794--2802.

\bibitem{brockman2016openai}
G.~Brockman, V.~Cheung, L.~Pettersson, J.~Schneider, J.~Schulman, J.~Tang, and
  W.~Zaremba, ``Openai gym,'' \emph{arXiv preprint arXiv:1606.01540}, 2016.

\bibitem{clavera2019model}
I.~Clavera, Y.~Fu, and P.~Abbeel, ``Model-augmented actor-critic:
  Backpropagating through paths,'' in \emph{International Conference on
  Learning Representations}, 2019.

\bibitem{morgan2021model}
A.~S. Morgan, D.~Nandha, G.~Chalvatzaki, C.~D’Eramo, A.~M. Dollar, and
  J.~Peters, ``Model predictive actor-critic: Accelerating robot skill
  acquisition with deep reinforcement learning,'' in \emph{2021 IEEE
  International Conference on Robotics and Automation (ICRA)}.\hskip 1em plus
  0.5em minus 0.4em\relax IEEE, 2021, pp. 6672--6678.

\end{thebibliography}

\end{document}